\def\year{2019}\relax
\documentclass[letterpaper]{article} 

\usepackage{aaai19}  
\usepackage{times}  
\usepackage{helvet}  
\usepackage{courier}  
\usepackage{url}  
\usepackage{graphicx}  

\usepackage{amssymb}
\usepackage{amsmath}

\usepackage{mathtools}
\usepackage{multicol}
\usepackage{float}

\usepackage{tikz}

\usepackage{pgfplots}

\usepackage{algorithm}
\usepackage{algorithmicx}
\usepackage[noend]{algpseudocode}

\usepackage[inline, shortlabels]{enumitem}
\setlist{nolistsep, leftmargin=*}

\usepackage{macros_scheduling}

\usepackage[font=scriptsize,labelfont=bf]{caption}

\begin{document}

\title{Argumentation for Explainable Scheduling\\
(Full Paper with Proofs)}
\author{Kristijonas \v Cyras\and Dimitrios Letsios\and Ruth Misener\and Francesca Toni \\
Imperial College London, London, UK}

\nocopyright

\maketitle

\begin{abstract}
Mathematical optimization offers highly-effective tools for finding solutions for problems with well-defined goals, notably scheduling. 
However, optimization solvers are often unexplainable black boxes
whose solutions are inaccessible to users and which users cannot interact with. 
We define a novel paradigm using argumentation to empower the interaction between optimization solvers and users, 
supported by tractable explanations which certify or refute solutions. 
A solution can be from a solver or of interest to a user (in the context of 'what-if' scenarios).
Specifically, we define argumentative and natural language explanations  for why a schedule is (not) feasible, (not) efficient or (not) satisfying fixed user decisions, 
based on models of the fundamental makespan scheduling problem in terms of abstract argumentation frameworks (AFs). 
We define three types of AFs, whose stable extensions are in one-to-one correspondence with schedules that are feasible, efficient and satisfying fixed decisions, respectively. 
We extract the argumentative explanations from these AFs and the natural language explanations from the argumentative ones.
\end{abstract}

\section{Introduction} 
\label{sec:intro}


Computational optimization empowers effective decision making.
Given a mathematical optimization model with well-defined numerical variables, objective function(s), and constraints,
a solver generates an efficient and ideally optimal solution.
If the model and solver are correct, then implementing the optimal solution can have major benefits.
But how can we explain the optimal solutions to a user?
Currently, solvers express necessary and sufficient optimality conditions with formal 
mathematics, so users often consider the optimization solver as an unexplainable 
black box.

\emph{Explainable scheduling} is a critical application \cite{Sacchi.et.al:2015} and our test
bed for \emph{explainable optimization}. 
Consider the fundamental makespan scheduling problem, a discrete optimization problem for effective resource allocation \cite{Graham1969}. 
This problem arises in for example nurse rostering where staff of different skill qualification categories, 
e.g.~Registered Nurse, Nurse's Aide, need to be assigned to shifts \cite{warner-prawda}. 
In the planning period, staff are scheduled, e.g.~for the next 4 weeks \cite{Burke2004}. 
But nursing personnel, hospital managers, or patients may inquire about the fairness or optimality of the schedule and possible changes. 
Further, when unexpected events occur, e.g.~staff illness or an unusually high influx of patients, feasible schedules must be recovered \cite{MOZ2007667}.
We take the first steps towards enabling users to interact with, and obtain explanations from, 
optimal scheduling in general and nurse rostering in particular.

This paper proposes a novel paradigm (that we call \emph{\argopt}) for explaining why a solution is (not) good by leveraging abstract argumentation (AA) as an intermediate layer between the user and optimization software.
Argumentation is highly suitable for explaining reasoning and decisions \cite{Moulin.et.al:2002,Atkinson.et.al:2017}
with argumentative explanations proposed in various settings, see e.g.~\cite{Garcia:Chesnevar:Rotstein:Simari:2013,Fan:Toni:2015,Cyras:Fan:Schulz:Toni:2018,Zeng.et.al:2018-AAMAS}. 
We show how AA offers computational optimization an accessible knowledge representation tool, namely AA frameworks (AFs), for modeling optimization problems. 
These AFs are constructed automatically given a scheduling problem instance and possibly fixed user decisions, 
allow formal explanation definitions and enable efficient generation of natural language explanations.
Figure \ref{fig:approach} illustrates \argopt.

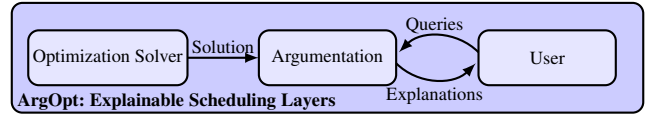
\begin{figure}

{\center
\begin{tikzpicture}[scale=0.73]

\draw[black, fill=blue!20, thick, rounded corners] (-0.5,5.5) rectangle (11,7.5) 
node[pos=0.5, black, opacity=1] {};

\node at (2.5,5.7) {\scriptsize \textbf{\argopt: Explainable Scheduling Layers}};

\draw[black, fill=blue!10, thick, rounded corners] (-0.2,6) rectangle (2.7,7) 
node[pos=0.5, black, opacity=1] {\scriptsize Optimization Solver};

\draw [-latex, thick] (2.7,6.5) -- (4,6.5);
\node at (3.35, 6.7) {\scriptsize Solution};

\draw[black, fill=blue!10, thick, rounded corners] (4,6) rectangle (6.5,7) 
 node[pos=0.5, black, opacity=1]{\scriptsize Argumentation};

\draw[black, fill=blue!10, thick, rounded corners] (8,6) rectangle (10.5,7) 
node[pos=0.5, black, opacity=1] {\scriptsize User};

\draw [-latex, thick] (8,6.6) to [bend right = 45] (6.55,6.6);
\node at (7.2, 7.1) {\scriptsize Queries};

\draw [-latex, thick] (6.5,6.4) to [bend right = 45] (7.95,6.4);
\node at (7.2, 5.8) {\scriptsize Explanations};

\end{tikzpicture}
}
\caption{\scriptsize \argopt\ produces explanations about solutions generated by an optimization solver to the fundamental makespan scheduling problem.
\emph{Argumentation} is an intermediate layer between the \emph{optimization solver} and the \emph{user}. 
The optimization solver passes the computed solution to the argumentation layer.
The user interacts with the argumentation layer to obtain argumentative and natural language explanations.}
\label{fig:approach}
\end{figure}

What makes an optimization solution good?
A good solution should 
\begin{enumerate*}[(i)]
\item be feasible, 
\item be efficient (ideally optimal), and 
\item satisfy fixed (user) decisions. 
\end{enumerate*}
\argopt\ introduces a toolkit realizing these needs and dealing with a number of the relevant challenges.
A good explanation should be efficiently attainable, combine few causal relationships
and admit simple natural language interpretations. 
To build trust, explanations should be associated with a formal representation providing interpretable certificates on why the explanation is valid and how it is generated. 
For tractability, we implement polynomial explainability and thereby achieve both computational
tractability, i.e.\ quick generation of results, and
cognitive tractability, i.e.\ clear user explanations.
Our explanation-generating engine has a modular structure for generating different types of explanations.

Given a problem instance, 
we construct AFs to explain problem instance solutions. 
We map decisions (schedule assignments) to arguments and capture feasibility and optimality conditions via attacks. 
We then extract from AFs argumentative explanations pertaining to the decisions and the related conditions, 
which can in turn be rendered in natural language. 
\argopt\ comprising an optimization solver and an argumentation layer 
can thereby justify its solutions against human-proposed solutions and effectively perform human-AI interaction for efficient decision making. 
We may derive explanations on potential infeasibilities and weak solutions.
Overall, explainability enables the decision maker to
\begin{enumerate*}[(i)]
\item check feasibility of possible solutions, 
\item perform what-if analysis for scenarios, and
\item recover feasible solutions after various disturbances. 
\end{enumerate*}

\section{Background}
\label{sec:background}


\subsection{Makespan Scheduling}
\label{subsec:scheduling}

An instance $I$ of the \emph{makespan scheduling problem}, e.g.\ \cite{Graham1969,Leung2004,Brucker2007}, is a pair $(\M, \J)$, where $\mathcal{J}=\{1,\ldots,n\}$ is a set of $n$ \emph{independent} jobs with a vector $\vec{p}=\{p_1,\ldots,p_n\}$ of processing times which have to be executed by a set $\mathcal{M}=\{1,\ldots,m\}$ of $m$ parallel identical machines.
Job $j\in\mathcal{J}$ must be processed by exactly one machine $i\in\mathcal{M}$ for $p_j$ units of time non-preemptively, i.e.\ in a single continuous interval without interruptions.
Each machine may process at most one job per time.
The objective is to find a minimum makespan schedule, i.e.\ to minimize the last machine completion time.
In the nurse rostering example, each task (or job) needs to be assigned a specific nurse (or machine). 
In this simple setting, each task is assigned to just one nurse and we aim for all staff to complete as soon as possible. 

In a standard mixed integer linear programming formulation, binary decision variable $x_{i, j}$ is 1 if job $j \in \J$ is executed by machine $i \in \M$ and 0 otherwise. 
Thus, a schedule of $(\M, \J)$ can be seen as an $m \times n$ matrix $S \in \{ 0, 1 \}^{m \times n}$ with entries $x_{i, j} \in \{ 0, 1 \}$ representing job assignments to machines, for $i \in \M$ and $j \in \J$.

Given a schedule $S$, let $C_i$ be the completion time of machine $i \in \M$ in $S$ 
and let $C_{\max}=\max_{1\leqslant i\leqslant m} \{ C_i \}$ be the makespan.
The problem is formally described by 
Equations (\ref{Eq:Makespan_Objective})--(\ref{Eq:ConstraintInteger}) (next column). 
Expression (\ref{Eq:Makespan_Objective}) minimizes makespan.
Constraints (\ref{Eq:Constraint_Max_Completion}) are the makespan definition.
Constraints (\ref{Eq:ConstraintMachineCompletionTime}) allow a machine to execute at most one job per time.
Constraints (\ref{Eq:ConstraintJobAssignment}) assign each job to exactly one machine.
An optimal schedule satisfies all Constraints (\ref{Eq:Constraint_Max_Completion})--(\ref{Eq:ConstraintJobAssignment}) and minimizes makespan objective Expression (\ref{Eq:Makespan_Objective}).

\begin{subequations}
\label{Eq:Makespan_MILP}
\begin{align}
\min_{C_{\max},C_i,x_{i,j}} \quad & C_{\max} \label{Eq:Makespan_Objective} \\ 
& C_{\max} \geqslant C_i && i\in\M \label{Eq:Constraint_Max_Completion} \\
& C_i = \sum_{j=1}^nx_{i,j}\cdot p_j && i\in\M \label{Eq:ConstraintMachineCompletionTime} \\
& \sum_{i=1}^m x_{i,j} = 1 && j\in\J \label{Eq:ConstraintJobAssignment} \\
& x_{i,j}\in\{0,1\} && j\in\J, i\in\M \label{Eq:ConstraintInteger}
\end{align}
\end{subequations}

In the nurse rostering example, this formulation allows to deal with skill qualifications, e.g.\ to limit tasks assigned to nurses. 
More elaborate nurse rostering incorporates contractual obligations, 
e.g.\ assigning a nurse the correct number of shifts per week, 
and allows holiday, e.g.\ avoiding jobs for a nurse in a given week.

This paper assumes 
an instance $I = (\M, \J)$ of a makespan scheduling problem with $\M = \{ 1, \ldots, m \}$ and $\J = \{ 1, \ldots, n \}$, for $m, n \geqslant 1$, unless stated otherwise.


\subsection{Abstract Argumentation (AA)}
\label{subsec:aa}

We give essential background on Abstract Argumentation (AA) following its original definition in \cite{Dung:1995}.

An \emph{AA framework} (\emph{AF}) is a directed graph $\AF$ with
\begin{itemize}
\item a set $\Args$ of \emph{arguments}, and
\item a binary \emph{attack} relation $\attacks$ over $\Args$.
\end{itemize}
For $\arga, \argb \in \Args$, $\arga \attacks \argb$ means that $\arga$ attacks $\argb$, 
and $\arga \nattacks \argb$ means that $\arga$ does not attack $\argb$. 
With an abuse of notation, we extend the attack notation to sets of arguments as follows.
For $\argsA \subseteq \Args$ and $\argb \in \Args$: 
\begin{itemize}
\item $\argsA \attacks \argb$ iff $\exists \arga \in \argsA$ with $\arga \attacks \argb$;
\item $\argb \attacks \argA$ iff $\exists \arga \in \argsA$ with $\argb \attacks \arga$;
\end{itemize}

A set $\argsE \subseteq \Args$ of arguments, also called an \emph{extension}, is
\begin{itemize}
\item \emph{conflict-free} iff $\argsE \nattacks \argsE$;
\item \emph{stable} iff 
$\argsE$ is conflict-free and $\forall \argb \in \Args \setminus \argsE$, $\argsE \attacks \argb$. 
\end{itemize}


While finding stable extensions of a given AF is NP-hard, verifying if a given extension is stable is polynomial (in the number of arguments) \cite{Dunne:Wooldridge:2009}.

\section{Setting the Ground for \argopt}
\label{sec:problem}


\label{subsec:expl}

Within the makespan scheduling problem, we identify three dimensions, namely 
schedule {\bf feasibility}, schedule {\bf efficiency} and accommodating {\bf fixed user decisions} within schedules, formally defined below.  
Our novel paradigm \argopt\ focuses on explaining why a given schedule:
\begin{itemize}
\item is \textbf{feasible} or not, 
\item is \textbf{efficient} or not, 
\item \textbf{satisfies fixed user decisions} or not. 
\end{itemize}

Feasibility simply amounts to dropping the makespan minimization objective:

\begin{definition}
\label{defn:feasibility}
A schedule is \emph{feasible} iff it meets 
constraints (\ref{Eq:Constraint_Max_Completion}) -- (\ref{Eq:ConstraintJobAssignment}). 
\end{definition}

It can be shown that this notion of feasibility amounts to assigning each job to exactly one machine:

\begin{lemma}
\label{lemma:feasibility}
A schedule is feasible iff it meets constraint (\ref{Eq:ConstraintJobAssignment}), 
i.e.\ $\sum_{i=1}^{m} x_{i, j} = 1$  $\forall j \in \J$.
\end{lemma}

\begin{proof}
If $\sum_{i=1}^{m} x_{i, j} = 1$  $\forall j \in \J$ holds, we observe that there is a trivial assignment of values to auxiliary variables $C_i$ $\forall i\in M$ and $C_{\max}$ which are used to model the objective function of the makespan scheduling problem instance.
It suffices to set $C_i=\sum_{j=1}^m x_{i,j}p_j$ $\forall i\in\M$ and $C_{\max}=\max_{i=1}^n\{\sum_{j=1}^n x_{i,j}p_j\}$.
\end{proof}

Feasibility is polynomial, whereas finding optimal solutions for the makespan scheduling problem is strongly NP-hard \cite{Garey1979}. 
A standard, less drastic approach in optimization to deal with intractability 
is approximation algorithms, 
e.g.\ the common longest processing time first heuristic \cite{Graham1969} produces a 4/3-approximate schedule, namely attaining makespan a constant factor far from the optimal.
In this vein, we define efficiency as feasibility  and satisfaction of  the single and pairwise exchange
properties that guarantee approximately optimal schedules.

\begin{definition}
\label{defn:optimality}
Schedule $S$ satisfies the \emph{single} and \emph{pairwise} \emph{exchange properties} iff 
for every \emph{critical} job $j \in \J$ such that $x_{i, j} = 1$ and $C_i = C_{\max}$, 
it respectively holds that, for $i' \neq i$,
\begin{itemize}
\item \emph{Single Exchange Property} (\emph{SEP}): 
$C_i - C_{i'} \leqslant p_j$;

\item \emph{Pairwise Exchange Property} (\emph{PEP}): 
for any job $j' \neq j$ with $x_{i', j'} = 1$, 
if $p_j > p_{j'}$, then $C_i + p_{j'} \leqslant C_{i'} + p_j$.
\end{itemize}
$S$ is \emph{efficient} iff $S$ is feasible and satisfies both SEP and PEP. 
\end{definition}

SEP concerns improving a schedule by a single exchange of any critical job between machines. 
PEP concerns pairwise exchanges of critical jobs with other jobs on other machines. 
SEP and PEP are necessary optimality conditions 
(but not sufficient; see e.g.\ 
the list-scheduling algorithm tightness analysis in \cite{Williamson2011}).

\begin{lemma}
\label{lemma:optimality}
Every optimal schedule satisfies SEP and PEP. 
\end{lemma}

\begin{proof}
\textbf{Notation:}
Given a schedule $S$, we may write $C_i(S)$ and $C_{\max}(S)$ to denote the completion time of machine $i$ in $S$ and the makespan of $S$, respectively. 

Let $S^*$ be an optimal schedule. 

SEP. 
Consider the schedule $S$ obtained by moving job $j$ from machine $i$ to $i'$ and keeping the assignments of the remaining jobs as in schedule $S^*$.
Clearly, $C_i(S)=C_i(S^*)-p_j$ and $C_{i'}(S)=C_{i'}(S^*)+p_j$.
Furthermore, $C_{i''}(S)=C_{i''}(S^*)$, for each $i''\neq i,i'$.
Because job $j$ is critical, $C_i(S^*)\geqslant C_{i'}(S^*)$.
Since $S^*$ is optimal, it must be the case that $S$ does not attain lower makespan, i.e.\ $C_{i'}(S)\geqslant C_i(S^*)$, or equivalently $C_i(S^*)-C_{i'}(S^*)\leqslant p_j$.

PEP. 
Consider the schedule $S$ obtained by moving job $j$ from machine $i$ to $i'$, moving job $j'$ from machine $i'$ to $i$ and keeping the assignments of the remaining jobs as in schedule $S^*$.
Clearly, $C_i(S)=C_i(S^*)-p_j+p_{j'}$ and $C_{i'}(S)=C_{i'}(S^*)+p_j-p_{j'}$.
Furthermore, $C_{i''}(S)=C_{i''}(S^*)$, for each $i''\neq i,i'$.
Because job $j$ is critical, $C_i(S^*) \geqslant C_{i'}(S^*)$.
Since $S^*$ is optimal, it must be the case that $S$ does not attain lower makespan, i.e.\ $\max\{ C_i(S), C_{i'}(S) \} \geqslant C_i(S^*)$. 
That is, either $C_i(S) \geqslant C_i(S^*)$, or $C_{i'}\geqslant C_i(S^*)$.
Equivalently, either $p_j \leqslant p_{j'}$, or $p_j - p_{j'} \geqslant C_i(S^*)-C_{i'}(S^*)$.
\end{proof}

The overall setting for explanation is as follows. 
An optimization solver recommends an optimal schedule $S^*$ for the given instance $I$ of the makespan scheduling problem. 
The user inquires whether another schedule $S$ could be used instead.
$S$ expresses changes to $S^*$ within `what-if' scenarios (e.g.\ `what if this job were assigned to that machine instead'?)
If $S$ is also optimal, then the answer is positive and a certifiable explanation as to why this is so may be provided. 
Else, if $S$ is (provably) not optimal, then the answer is negative and an explanation is generated as to why. 

In addition, we envisage that a user may fix some decisions, i.e.\ (non-)assignments of jobs to machines  originally 
unbeknownst to the scheduler, e.g.\ that a nurse is absent or lacks necessary skills for a task, or that a nurse volunteers for a task,
and may want to find out whether and why $S$ satisfies or violates these decisions.


\begin{definition}
\label{defn:fd}
Let
\begin{itemize}
\item \emph{negative fixed decisions} be $D^- = \M^- \times \J^- \subseteq \M \times \J$; 
\item \emph{positive fixed decisions} be $D^+ = \M^+ \times \J^+ \subseteq \M \times \J$; 
\item \emph{fixed decisions} be $D = (D^-, D^+)$ such that $D^- \cap D^+ = \emptyset$ and $\nexists (i, j), (k, j) \in D^+$ with $i \neq k$. 
\end{itemize}
We say that schedule $S$ satisfies 
\begin{itemize}
\item $D^-$ iff $(i, j) \in D^-$ implies $x_{i, j} = 0$;
\item $D^+$ iff $(i, j) \in D^+$ implies $x_{i, j} = 1$;
\item $D = (D^-, D^+)$ iff $S$ satisfies both $D^-$ and $D^+$.
\end{itemize}
$S$ violates $D^-$, $D^+$, $D$ iff it does not satisfy $D^-$, $D^+$, $D$, resp.
\end{definition}

Negative (resp.\ positive) fixed decisions insist on which jobs cannot (resp.\ must) be done on which machines.
Fixed decisions consist of compatible negative and positive fixed decisions, 
where the positive fixed decisions, if any, are feasible in that no two machines are asked to do the same job. 

Note that capturing fixed decisions allows us to capture various phenomena, such as (in the running example):
\begin{itemize}
\item nurse $i$ falling ill, with $D_i^- = \{ (i, j)~:~j \in \J \}$;
\item cancelled job $j$, with $D_j^- = \{ (i, j)~:~i \in \M \}$.
\end{itemize}

These may be particularly useful in a dynamic setting, where fixed decisions need to be taken into account after having executed some part of the schedule. 

In summary, feasibility is a basic constraint; 
efficiency concerns schedule-specific optimality conditions; 
fixed user decisions pertain to schedule-specific feasibility while disregarding optimality. 

Our paradigm \argopt\ is driven by the following desiderata (where `good' stands for any amongst 
feasible, efficient or satisfying fixed decisions):

\begin{itemize}
\item \textbf{soundness and completeness} of explanations, in the sense that, given schedule $S$, 
there exists an explanation as to why $S$ is (not) `good' 
iff $S$ is (not) `good'; 

\item 
\textbf{computational tractability}, in the sense that explaining whether and why schedule $S$ is (not) `good' can be performed in polynomial time in the size of the makespan scheduling problem instance $I$; 

\item 
\textbf{cognitive tractability}, in the sense that each
explanation pertaining to schedule $S$ and presented to the user should be polynomial in the size of $S$.
\end{itemize}

Tractability is crucial to ensure that explaining results in a low construction overhead, 
an essential property of explanations \cite{Sormo:Cassens:Aamodt:2005}. 
We have restricted attention to feasible and efficient (rather than optimal) solutions because
answering and 
explaining why an arbitrary schedule is optimal in polynomial time is ruled out due to NP-hardness of makespan scheduling, unless P$=$NP. 

We choose {\em argumentation} as the underlying technology for \argopt\ as it serves us well to fulfil the above desiderata:

\begin{itemize}
\item argumentation affords knowledge representation tools, such as AFs,
for providing sound and complete counterparts for a diverse range of problems, 
e.g.\ games  \cite{Dung:1995},
and it has long been identified as suitable for explaining, see e.g.\ \cite{Moulin.et.al:2002,Atkinson.et.al:2017}; 
we define counterparts for determining `good' schedules (see Section~\ref{sec:integration}) and build upon these for defining sound and complete explanations (see Section~\ref{sec:expl});

\item
argumentation enables tractable 
specifications of the optimization problems we consider and tractable generation of explanations (see Section~\ref{sec:expl}); 

\item
cognitively tractable explanations can be extracted from AFs, acting as {certificates as to why a schedule is (not) `good'} (see Section~\ref{sec:expl}).
\end{itemize}

\section{Argumentation for Optimization}
\label{sec:integration}


We approach the issue of explaining, using AA, whether and why the solutions to the makespan scheduling problem are `good' in three steps, as illustrated in Figure \ref{fig:components}. 

\begin{figure}[h]

\centering
\begin{tikzpicture}[scale=0.73]

\draw[black, fill=blue!10, thick, rounded corners] (0,5.5) rectangle (2.5,6.5) 
node[pos=0.5, black, opacity=1] {\scriptsize Instance $I$};

\draw[black, fill=blue!10, thick, rounded corners] (3,5.5) rectangle (7.5,6.5) 
node[pos=0.5, black, opacity=1, align=center]{\scriptsize Schedule $S$ \\
\scriptsize from the optimizer or the user};

\draw[black, fill=blue!10, thick, rounded corners] (7.85,5.5) rectangle (10.65,6.5) 
node[pos=0.5, black, opacity=1, align=center] {\scriptsize Fixed Decisions $D$ \\
\scriptsize from the user};

\draw[black, fill=blue!20, thick, rounded corners] (-0.5,3) rectangle (11,5) 
node[pos=0.5, black, opacity=1] {};
\node at (5.53,3.2){\scriptsize \textbf{Argumentation Layer Components}};
\draw[-latex, thick] (1.25,5.5) -- (1.25,4.5); 

\draw[black, fill=blue!10, thick, rounded corners] (0,3.5) rectangle (2.5,4.5) 
node[pos=0.5, black, opacity=1] {\scriptsize Feasibility AF};
\draw[-latex, thick] (5.25,5.5) -- (5.25,4.5); 

\draw[black, fill=blue!10, thick, rounded corners]
(4,3.5) rectangle (6.5,4.5) node[pos=0.5, black, opacity=1] {\scriptsize Optimality AF};
\draw[-latex, thick] (9.25,5.5) -- (9.25,4.5); 

\draw[black, fill=blue!10, thick, rounded corners] (7.85,3.5) rectangle (10.65,4.5) 
node[pos=0.5, black, opacity=1] {\scriptsize Fixed Decision AF };
\draw[-latex, thick, dashed] (2.5, 4) to (4, 4); 
\draw[-latex, thick, dashed] (6.5, 4) to (7.85, 4); 

\draw[black, fill=blue!10, thick, rounded corners] (3.6,1.5) rectangle (6.9,2.5) 
node[pos=0.5, black, opacity=1, align=center] { \scriptsize Explanations \\
\scriptsize to the user};

\draw[-latex, thick] (1.25,3.5) -- (3.7, 2.5); 
\draw[-latex, thick] (5.25,3.5) -- (5.25, 2.5); 
\draw[-latex, thick] (9.25,3.5) -- (6.8, 2.5); 

\end{tikzpicture}
\caption{\scriptsize Argumentation layer components in \argopt:
\begin{enumerate*}[(i)]
\item the feasibility AF $\AFF$ takes an instance $I$ 
and explains whether a given schedule for $I$ is feasible;
\item the optimality AF $\AFS$ takes the instance $I$ represented via $\AFF$ and a schedule $S$ for $I$, and explains whether $S$ is efficient; 
\item the fixed decision AF $\AFD$ takes either $\AFF$ with a schedule or $\AFS$, some fixed decisions $D$ and explains whether the schedule satisfies these decisions.
\end{enumerate*}}
\label{fig:components}
\end{figure}

First, we capture schedule feasibility in the sense of mapping an instance of the makespan scheduling problem into an AF by identifying assignments with arguments, 
so that the feasible schedules are in one-to-one correspondence with the stable extensions. 
We then capture optimality conditions in the sense of mapping the properties of a given schedule into a schedule-specific AF by modifying the attack relation, 
so that the schedule satisfying optimality conditions equates to the corresponding extension being stable. 
Lastly, we capture fixed decisions for a specific schedule in the absence of optimality considerations, 
in the sense of mapping the fixed decisions into an AF by modifying the attacks, 
so that the schedule satisfying fixed decisions equates to the corresponding extension being stable.  

The design of the proposed AFs incorporates tractability. 
First, the mappings from the makespan scheduling problem to AA are polynomial.
Second, 
explaining whether and why a given schedule is feasible and/or efficient (optimality) and/or satisfies fixed user decisions amounts to \emph{verifying} if the corresponding stable extension is stable in the relevant AF; 
this problem is also polynomial. 


We chose stable extensions as the underlying semantics for the following reasons. 
\begin{enumerate*}[1.]
\item The makespan scheduling problem requires all jobs to be assigned, 
so we need a ``global'' semantics \cite{Dung:1995} 
and stable semantics is one such.

\item Verification of stable extensions is polynomial, allowing us to meet the computational tractability desideratum.

\item Other semantics are either non-global 
(e.g.\ the grounded extension)
or have non-polynomial verification 
(e.g.\ $coNP$-complete for preferred extensions). 

\end{enumerate*}


\subsection{Feasibility}
\label{subsec:feasibility}

We model the feasibility space of makespan scheduling in AA to be able to explain why a user's proposed schedule is or not feasible.
We do this by mapping assignment variables (binary decisions) to arguments and capturing feasibility constraints via attacks. 
Specifically, argument $\arga_{i,j}$ stands for an assignment of job $j \in \J$ to machine $i \in \M$.
Attacks among arguments model pairwise incompatible decisions:
$\arga_{i,j}$ and $\arga_{k, l}$ attacking each other models the incompatibility of assignments $\arga_{i, j}$ and $\arga_{k, l}$.
Intuitively, the attack relation encodes different machines competing for the same job. 
Formally:

\begin{definition}
\label{defn:feasibility af}
The \emph{feasibility AF} $\AFF$ is given by
\begin{itemize}
\item $\ArgsF = \{ \arga_{i, j}~:~i \in \M, j \in \J \}$, 

\item $\arga_{i, j} \attacksF \arga_{k, l}$ ~ iff ~ $i \neq k$ and $j = l$.
\end{itemize}
\end{definition}

The following definition formalizes a natural bijective mapping between schedules and extensions. 

\begin{definition}
\label{defn:corresponding}
Let $\AFF$ be the feasibility AF. 
A schedule $S$ and an extension $\argsE \subseteq \ArgsF$ are \emph{corresponding}, denoted $S \approx \argsE$, when the following invariant holds:
\[
x_{i, j} = 1 ~ \text{iff} ~ \arga_{i, j} \in \argsE. 
\]
\end{definition}

With this correspondence, 
the feasibility AF encodes exactly the feasibility of the makespan scheduling problem, 
in that feasible schedules coincide with stable extensions:

\begin{theorem}
\label{thm:feasibility}
Let $\AFF$ be the feasibility AF. 
For any $S \approx \argsE$, $S$ is feasible iff $\argsE$ is stable. 
\end{theorem}

\begin{proof}
Let $\argsE$ be a stable extension of $\AFF$. 
To prove that the corresponding schedule $S$ is feasible, we need to show that Equation \ref{Eq:ConstraintJobAssignment} holds:
$\sum_{i=1}^{m} x_{i, j} = 1$ for any $j \in \J.$

As $x_{i, j} \in \{ 0, 1 \} ~ \forall i, j$, we have $\sum_{i=1}^{m} x_{i, j} \in \mathbb{N} \cup \{ 0 \}$. 
Suppose for a contradiction that for some $j \in \J$ we have $\sum_{i=1}^{m} x_{i, j} \neq 1$. 

\begin{enumerate}[a)]
\item First assume $\sum_{i=1}^{m} x_{i, j} > 1$. 
Then $\arga_{i, j}, \arga_{k, j} \in \argsE$ for some $i \neq k$. 
But then $\arga_{i, j} \attacksF \arga_{k, j}$,
whence $\argsE$ is not conflict-free.
This contradicts $\argsE$ being stable. 

\item Now assume $\sum_{i=1}^{m} x_{i, j} = 0$. 
Then $\arga_{i, j} \not\in \argsE ~ \forall i \in \M$. 
By definition of $\attacksF$, we then have $\argsE \nattacksF \arga_{i, j}$ for any $i \in \M$. 
In particular, $\argsE \nattacksF \arga_{1, j}$.  
This contradicts $\argsE$ being stable. 
\end{enumerate}

We next prove that if $S$ is a feasible schedule, then the corresponding extension $\argsE$ is  stable in $\AFF$. 

We have $\sum_{i=1}^{m} x_{i, j} = 1$ for any $j \in \J$ because $S$ is feasible. 
This means that for every $j \in \J$, $\argsE$ contains one and only one $\arga_{i, j}$ for some $i \in \M$. 
Thus, by definition of $\attacksF$, $\argsE$ is conflict-free. 
Moreover, for any $j \in \J$, $\arga_{i, j} \in \argsE$ attacks every other $\arga_{k, j}$ with $k \neq i$. 
Hence, $\argsE$ is stable in $\AFF$. 
\end{proof}


\begin{figure*}
\centering
\footnotesize

\begin{tikzpicture}
\node at (-0.6, 1.2) {a)};
\node at (0, 1.2) {$\arga_{1, 1}$};
\draw (0, 1.2) circle (3mm);
\node at (1.2, 1.2) {$\arga_{1, 2}$};
\draw (1.2, 1.2) circle (3mm);
\node at (0, 0) {$\arga_{2, 1}$};
\draw (0, 0) circle (3mm);
\node at (1.2, 0) {$\arga_{2, 2}$};
\draw (1.2, 0) circle (3mm);

\draw[latex-latex, thick] (0, 0.9) to (0, 0.3);
\draw[latex-latex, thick] (1.2, 0.9) to (1.2, 0.3);

\node at (2.4, 1.2) {b)};
\node at (3, 1.2) {$\arga_{1, 1}$};
\draw (3, 1.2) circle (3mm);
\node at (4.2, 1.2) {$\arga_{1, 2}$};
\draw (4.2, 1.2) circle (3mm);
\node at (5.4, 1.2) {$\arga_{1, 3}$};
\draw (5.4, 1.2) circle (3mm);

\node at (3, 0) {$\arga_{2, 1}$};
\draw (3, 0) circle (3mm);
\node at (4.2, 0) {$\arga_{2, 2}$};
\draw (4.2, 0) circle (3mm);
\node at (5.4, 0) {$\arga_{2, 3}$};
\draw (5.4, 0) circle (3mm);

\draw[latex-, thick] (3, 0.9) to (3, 0.3);
\draw[latex-latex, thick] (4.2, 0.9) to (4.2, 0.3);
\draw[latex-latex, thick] (5.4, 0.9) to (5.4, 0.3);

\draw[-latex, thick] (5.3, 0.27) to (4.45, 1);


\draw[-latex, thick, dashed, opacity = 0.3, out=240, in=120] (2.9, 0.9) to (2.9, 0.3); 

\draw[dashed] (2.6, 1.6) -- (4.4, 1.6) -- (5.8, 0.2) -- (5.8, -0.4) -- (5.2, -0.4) -- (4, 0.8) -- (2.6, 0.8) -- (2.6, 1.6);

\node at (6.4, 1.2) {c)};
\node at (7, 1.2) {$\arga_{1, 1}$};
\draw (7, 1.2) circle (3mm);
\node at (8.2, 1.2) {$\arga_{1, 2}$};
\draw (8.2, 1.2) circle (3mm);
\node at (7, 0) {$\arga_{2, 1}$};
\draw (7, 0) circle (3mm);
\node at (8.2, 0) {$\arga_{2, 2}$};
\draw (8.2, 0) circle (3mm);

\draw[-latex, thick] (7, 0.9) to (7, 0.3);
\draw[latex-latex, thick] (8.2, 0.9) to (8.2, 0.3);
\draw[-, out=30, in=90, thick] (8.5, 0.1) to (9, 0);
\draw[-latex, out=270, in=330, thick] (9, 0) to (8.5, -0.1);


\draw(6.6, 1.6) -- (8.6, 1.6) -- (8.6, 0.8) -- (6.6, 0.8) -- (6.6, 1.6);

\node at (9.4, 1.2) {d)};
\node at (10, 1.2) {$\arga_{1, 1}$};
\draw (10, 1.2) circle (3mm);
\node at (11.2, 1.2) {$\arga_{1, 2}$};
\draw (11.2, 1.2) circle (3mm);
\node at (10, 0) {$\arga_{2, 1}$};
\draw (10, 0) circle (3mm);
\node at (11.2, 0) {$\arga_{2, 2}$};
\draw (11.2, 0) circle (3mm);

\draw[latex-latex, thick] (10, 0.9) to (10, 0.3);
\draw[latex-latex, thick] (11.2, 0.9) to (11.2, 0.3);


\draw[dotted] (9.6, 1.6) -- (10.4, 1.6) -- (10.4, 0.8) -- (9.6, 0.8) -- (9.6, 1.6);
\draw[-latex, opacity=0.3, thick, dashed] (10.3, 1.2) to (10.9, 1.2);
\draw[-latex, opacity=0.3, thick, dashed] (10.25, 1) to (10.95, 0.2);

\node at (12.4, 1.2) {e)};
\node at (13, 1.2) {$\arga_{1, 1}$};
\draw (13, 1.2) circle (3mm);
\node at (14.2, 1.2) {$\arga_{1, 2}$};
\draw (14.2, 1.2) circle (3mm);
\node at (13, 0) {$\arga_{2, 1}$};
\draw (13, 0) circle (3mm);
\node at (14.2, 0) {$\arga_{2, 2}$};
\draw (14.2, 0) circle (3mm);

\draw[-latex, thick] (13, 0.9) to (13, 0.3);
\draw[latex-latex, thick] (14.2, 0.9) to (14.2, 0.3);
\draw[-, out=30, in=90, thick] (14.5, 0.1) to (15, 0);
\draw[-latex, out=270, in=330, thick] (15, 0) to (14.5, -0.1);


\draw[dotted] (14.2, 1.7) -- (14.7, 1.2) -- (13, -0.5) -- (12.5, 0) -- (14.2, 1.7);
\draw[-latex, opacity=0.3, thick, dashed] (13.9, 1.2) to (13.3, 1.2);
\end{tikzpicture}

\caption{
In all graphs depicting AFs, nodes hold arguments and edges hold attacks and (dashed) non-attacks.
\begin{enumerate*}[a)]
\item Feasibility AF of Example \ref{ex:feasibility}.
\item Optimality AF of Example \ref{ex:optimality}; 
(dashed) 
box highlights the extension (and the corresponding schedule) in question; 
in Example \ref{ex:natt optimal}, the non-attack 
(dashed) 
explains why the schedule is not near-optimal, particularly violates SEP. 
\item Fixed decision AF of Example \ref{ex:fixed}; 
the
box indicates the unique stable extension. 
\item Feasibility AF with non-attacks 
(dashed)
explaining why the schedule of Example \ref{ex:natt feasible}
(the corresponding extension of which is highlighted in the 
(dotted)
box) 
is not feasible.
%
%
\item Fixed decision AF with the non-attack
(dashed)
explaining why the schedule of Example \ref{ex:natt fixed}
(the corresponding extension of which is highlighted in the 
(dotted)
box) 
violates fixed decisions.
\end{enumerate*}
}
\label{fig:graphs}
\end{figure*}
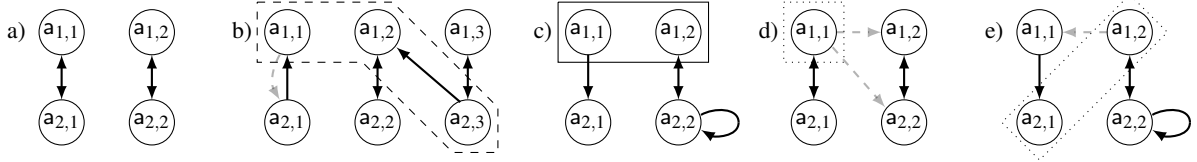

\begin{example}
\label{ex:feasibility}
Let $\M = \{ 1, 2 \}$, $\J = \{ 1, 2 \}$, e.g.\ 2 nurses for 2 tasks.
Figure \ref{fig:graphs}a) depicts the feasibility AF $\AFF$. 
It has 4 stable extensions:
$\{ \arga_{1, 1}, \arga_{1, 2} \}$, $\{ \arga_{1, 1}, \arga_{2, 2} \}$,
$\{ \arga_{2, 1}, \arga_{1, 2} \}$, $\{ \arga_{2, 1}, \arga_{2, 2} \}$.
They correspond to the 4 feasible $(\M, \J)$ rostering schedules where each task is completed by 1 nurse.
\end{example}

Finally, note that constructing the feasibility AF, as well as verifying if a given schedule is feasible, is polynomial:

\begin{lemma}
\label{lemma:feasibility polynomial}
The feasibility AF $\AFF$ can be constructed in $O(nm^2)$ time.
Verifying whether an extension $\argsE \subseteq \ArgsF$ is stable can be done in $O(n^2m^2)$ time.
\end{lemma}

\begin{proof}
The feasibility AF $\AFF$ consists of $O(nm)$ arguments.
In an adjacency list data structure, each argument $\arga_{i,j}\in\ArgsF$ attacks $O(m)$ arguments with respect to $\attacks_F$.
Thus, constructing $\AFF$ requires $O(nm^2)$ time.

Consider an extension $\argsE\subseteq\Args_F$.
We may check whether an extension $\argsE$ is conflict-free in $O(n^2m^2)$ time by determining whether $\arga_{i,j}\attacksF\arga_{k,\ell}$ and vice versa, for each pair of arguments $\arga_{i,j},\arga_{k,\ell}\in \argsE$. 
We also need $O(n^2m^2)$ time to examine if every argument $\argb\in \ArgsF\setminus \argsE$ is attacked by an argument $\arga\in \argsE$.
Hence, verifying whether the extension $\argsE$ is stable necessitates $O(n^2m^2)$ time.
\end{proof}

\subsection{Optimality Conditions}
\label{subsec:optimality}

We model optimality conditions in AA
to be able to explain why a user's proposed schedule is or not efficient. 
Lemma \ref{lemma:optimality} implies that if a feasible schedule $S$ can be improved by making a single exchange, i.e.~$S$ violates SEP, then $S$ is not optimal. 
Likewise, Lemma \ref{lemma:optimality} implies that if a feasible schedule $S$ can be improved by making a pairwise exchange, i.e.~$S$ violates PEP, then $S$ is not optimal.

We model both SEP and PEP in a single schedule-specific AF by modifying the feasibility AF as follows. 
Given $S$, we know $C_{\max}$, and so all the \emph{critical machines} $i$ (such that $C_i = C_{\max}$)
and all the associated \emph{critical jobs} $j$ (such that $x_{i, j} = 1$). 
Then, for any (critical) pair $(i, j)$ and any other machine $i' \neq i$, 
if $C_i > C_{i'} + p_j$, then $S$ violates SEP and can be improved by making the single exchange of moving job $j$ from machine $i$ to machine $i'$. 
We model this by \emph{removing} the attack $\arga_{i, j} \attacksF \arga_{i', j}$ from $\AFF$, 
and this represents that machine $i$ should not be competing for job $j$ with machine $i'$. 
Similarly, for any (critical) pair $(i, j)$ and any other machine $i' \neq i$ assigned some other job $j' \neq j$ with $p_j > p_{j'}$, 
if $C_i + p_{j'} > C_{i'} + p_j$, 
then $S$ violates PEP and can be improved by a pairwise exchange of $j$ and $j'$ from $i$ to $i'$. 
We model this by \emph{adding} an attack from $\arga_{i', j'}$ to $\arga_{i, j}$ in $\AFF$, 
and this represents that assigning $j'$ to $i'$ conflicts with assigning $j$ to $i$ because the latter is less efficient.

\begin{definition}
\label{defn:optimality af}
Let $\AFF$ be the feasibility AF and $S$ a schedule. 
The \emph{optimality AF} $\AFS$ is given by
\begin{itemize}
\item $\ArgsS = \ArgsF$, 
\item $\attacksS \, = \, \\
\left( \attacksF \setminus 
\{ (\arga_{i, j}, \arga_{i', j})~:~C_i = C_{\max}, x_{i, j} = 1, C_i > C_{i'} + p_j \} \right) \cup 
\{ (\arga_{i', j'}, \arga_{i, j})~:~C_i = C_{\max}, x_{i, j} = 1, x_{i', j'} = 1, i' \neq i, j' \neq j, \\
\text{ } \qquad\qquad\quad\;\;  p_j > p_{j'}, C_i + p_{j'} > C_{i'} + p_j \}
$.
\end{itemize}
We say that $\AFF$ is \emph{underlying} $\AFS$. 
\end{definition}

To determine whether the user's proposed schedule is efficient we just need to check if the corresponding extension is stable in the optimality AF:

\begin{theorem}
\label{thm:optimality}
Let $\AFF$ be the feasibility AF, $S$ a schedule and $S \approx \argsE$. 
Let $\AFS$ be the optimality AF. 
Then $\argsE$ is stable in $\AFS$ iff $S$ is feasible and satisfies both SEP and PEP. 
\end{theorem}

\begin{proof}
Let $\argsE$ be stable in $\AFS$. 
Then $\argsE$ is conflict-free in $\AFF$, 
because the attacks removed to capture SEP only make asymmetric attacks symmetric,
and the attacks added to capture PEP are among arguments not attacking each other in $\AFF$. 
For the same reason $\argsE$ is also stable in $\AFF$. 
Hence, $S$ is feasible, by Theorem \ref{thm:feasibility}. 
Moreover, as $\argsE$ is stable, $\forall j \in \J$ we find $\arga_{i, j} \in \argsE$ with $\arga_{i, j} \attacksS \arga_{i', j}~\forall i' \neq i$. 
This means no attacks were removed from $\attacksF$ to obtain $\attacksS$, 
so, in particular, 
$C_i = C_{\max}, ~ x_{i, j} = 1, ~ C_i > C_{i'} + p_j$
cannot hold for any $(i, i', j)$. 
Hence, $S$ satisfies SEP. 
Similarly, as $\argsE$ is conflict-free, we have that 
$C_i = C_{\max}, x_{i, j} = 1, x_{i', j'} = 1, i' \neq i, j' \neq j, p_j > p_{j'}, \break C_i + p_{j'} > C_{i'} + p_j$ cannot all hold for any $(i, i', j, j')$. 
Hence, $S$ satisfies PEP. 

If $S$ is feasible, then $\argsE$ is stable in $\AFF$. 
If $S$ satisfies SEP and PEP, then $\attacksS = \attacksF$. 
So if $S$ is feasible and satisfies SEP and PEP, then $\argsE$ is stable in $\AFS$ too. 
\end{proof}

\begin{example}
\label{ex:optimality}
Let $\M = \{ 1, 2 \}$, $\J = \{ 1, 2, 3 \}$ and
$p_1 = p_3 = 1$, $p_2 = 2$.
Let schedule $S$ be given by
$x_{1, 1} = x_{1, 2} = x_{2, 3} = 1$ and $x_{2, 1} = x_{2, 2} = x_{1, 3} = 0$,
e.g.\ nurse 1 completes jobs $1$ and $2$, and nurse 2 does job $3$.
$S$ is feasible with
$C_1 = x_{1, 1} p_1 + x_{1, 2} p_2 + x_{1, 3} p_3 = 3$ and $C_2 = 1$.
So nurse 1 is critical, i.e.\ serves the longest shift, and both jobs 1 and 2 are critical.
Since $C_1 = 3 > 2 = 1 + 1 = C_2 + p_1$ and $C_1 + p_3 = 4 > 3 = C_2 + p_2$,
schedule $S$ violates SEP and PEP.

So, given the feasibility AF $\AFF$,
construct the optimality AF $\AFS$ specific to $S$ with:
\begin{itemize}
\item $\ArgsS = \ArgsF = \{ \arga_{i, j}~:~i \in \{ 1, 2 \}, ~ j \in \{ 1, 2, 3 \} \}$,
\item $\attacksS \, = \, \{ (\arga_{i, j}, \arga_{i, l})~:~j \neq l \}                                                                      
\setminus \{ (\arga_{1, 1}, \arga_{2, 1}) \}                                                                                                
\cup                                                                                                                                        
\{ (\arga_{2, 3}, \arga_{1, 2}) \}$.
\end{itemize}
Figure \ref{fig:graphs}b) depicts the resulting $\AFS$.
The extension $\{ \arga_{1, 1}, \arga_{1, 2}, \arga_{2, 3} \} \approx S$ is not stable in $\AFS$ because (i)
it has conflict $\arga_{2, 3} \attacksS \arga_{1, 2}$ and (ii)
it does not attack $\arga_{2, 1}$.
\end{example}

As for feasibility, constructing the optimality AF, as well as verifying if a given schedule is efficient, is polynomial:

\begin{lemma}
\label{lemma:optimality polynomial}
Given a schedule $S$, the optimality AF $\AFS$ can be constructed in $O(nm^2)$ time.
Verifying whether an extension $\argsE\subseteq\ArgsS$, such that $\argsE \approx S$, is stable can be done in $O(n^2m^2)$ time. 
\end{lemma}

\begin{proof}
By Lemma \ref{lemma:feasibility polynomial}, constructing the feasibility AF $\AFF$ requires $O(nm^2)$ time.
To remove the attacks $\arga_{i,j}\attacksF\arga_{i',j}$ from $\AFF$, for each $i,i'\in\M$ and $j\in\J$ such that $C_i=C_{\max}$, $x_{i,j}=1$, and $C_i>C_{i'}+p_j$ in schedule $S$, 
we need $O(m)$ time to identify the critical machines and $O(nm^2)$ time to check the last inequality for each triplet $(i,i',j)\in\M\times\M\times\J$. 
To add the attacks $\arga_{i',j'}\attackspe\arga_{i,j}$ in $\AFF$, for each $i,i'\in\M$ and $j,j'\in\J$ such that $C_i=C_{\max}$, $x_{i,j}=1$, $x_{i',j'}=1$, $p_j>p_{j'}$, and $C_i+p_{j'}>C_{i'}+p_j$ in schedule $S$, 
we need $O(m)$ time to identify the critical machines and $O(n^2m^2)$ time to check whether any quadruple $(i,i',j,j')\in\M\times\M\times\J\times\J$ satisfies the last inequality. 

We can determine whether an extension $\argsE\in\ArgsS$ is stable in $O(n^2m^2)$ time, similarly to the case of the feasibility AF $\AFF$.
In particular, we may check if $\argsE$ is conflict-free in $O(n^2m^2)$ time, and if $\argsE$ attacks every argument of $\ArgsS\setminus \argsE$ in $O(n^2m^2)$ time. 
\end{proof}

\subsection{Fixed User Decisions}
\label{subsec:fixed}

We model fixed user decisions to be able to explain why a given schedule satisfies or violates the given fixed decisions. 
We do this, given either a feasibility or an optimality AF, by modifying the attacks in the (underlying) feasibility AF. 
Specifically, a negative decision induces a self-attack on the respective argument, 
whereas a positive decision results in removal of all the attacks on the respective argument.

\begin{definition}
\label{defn:fixed}
Let $\AFF$ be the feasibility AF, possibly underlying a given optimality AF $\AFS$, 
and let $D = (D^-, D^+)$ be fixed decisions. 
The \emph{fixed decision AF} $\AFD$ is given by 
\begin{itemize}
\item $\ArgsD = \ArgsF$, 
\item $\attacksD \, = \, \left( \attacksF \cup \{ (\arga_{i, j}, \arga_{i, j})~:~(i, j) \in D^- \} \right)
\setminus \\
\text{ } \qquad\;\; \{ (\arga_{k, l}, \arga_{i, j})~:~(i, j) \in D^+, (k, l) \in \M \times \J \}$.
\end{itemize}
\end{definition}

Fixed decisions thus result in the respective arguments being self-attacked or unattacked. 
This yields the following. 

\begin{theorem}
\label{thm:fixed}
Let $S$ be a schedule, $\AFF$ the feasibility AF, possibly underlying a given optimality AF $\AFS$, $\argsE \approx S$, $D$ a fixed decision, $\AFD$ the fixed decision AF. 
Then $S$ is feasible and satisfies $D$ iff $\argsE$ is stable in $\AFD$. 
\end{theorem}

\begin{proof}
If $S$ is feasible, then $\argsE$ is stable in $\AFF$, by Theorem \ref{thm:feasibility}.
So if $S$ satisfies $D$, then no $\arga_{i, j} \in \argsE$ attacks itself in $\AFD$, so $\argsE$ is conflict-free in $\AFD$. 
Furthermore, since $S$ satisfies $D$, 
we know that any $\arga_{i, j} \in \argsE$ with $(i, j) \in D^+$ attacks all $\arga_{k, j} \not\in \argsE$ with $k \neq i$. 
Hence, $\argsE$ is stable in $\AFD$. 

If $\argsE$ is stable in $\AFD$, then $\argsE$ is stable in $\AFF$ too, by definition of $D$, so $S$ is feasible, by Theorem \ref{thm:feasibility}.
Also, $\argsE$ being conflict-free in $\AFD$ implies that $\arga_{i, j} \not\in \argsE~\forall (i, j) \in \M' \times \J'$, 
so that $x_{i, j} = 0~\forall (i, j) \in \M' \times \J'$, whence $S$ satisfies $D^-$. 
Also, $\argsE$ must contain all the unattacked arguments in $\AFD$. 
Since $(i, j) \in D^+$ implies that $\arga_{i, j}$ is unattacked in $\AFD$, 
it also implies $\arga_{i, j} \in \argsE$, whence $x_{i, j} = 1$, as $\argsE \approx S$. 
That is, $S$ satisfies $D^+$ too, and so satisfies $D$.
\end{proof}

\begin{example}
\label{ex:fixed}
In Example \ref{ex:feasibility}, let $D^- = \{ (2, 2) \}$ and $D^+ = \{ (1, 1) \}$, 
e.g.\ nurse 2 cannot do job 2 and nurse 1 must do job 1.
Then $D = (D^-, D^+)$.
The fixed decision AF $\AFD$ is depicted in Figure \ref{fig:graphs}c).
It has a unique stable extension $\{ \arga_{1, 1}, \arga_{1, 2} \}$,
which corresponds to the unique feasible schedule satisfying $D$, in which nurse 1 is assigned both jobs.
\end{example}

Clearly, modeling fixed decisions is polynomial:

\begin{lemma}
\label{lemma:fixed polynomial}
Given a schedule $S$, the fixed decision AF $\AFD$ can be constructed in $O(n^2m^2)$ time.
Verifying whether the extension $\argsE\subseteq\ArgsD$, such that $\argsE \approx S$, is stable can be done in $O(n^2m^2)$ time. 
\end{lemma}

\begin{proof}
By Lemma \ref{lemma:feasibility polynomial}, constructing the feasibility AF $\AFF$ requires $O(nm^2)$ time.
To derive the fixed decision AF $\AFD$, we add the attack $\arga_{i,j}\attacksD\arga_{i,j}$ $\forall(i,j)\in D^-$ and we remove the attack $\arga_{k,l}\attacks_F\arga_{i,j}$ from $\AFF$, $\forall (i,j)\in D^+$ and $(k,l)\in\M\times\J$.
Specifically, we need $O(nm)$ time for attack additions and $O(n^2m^2)$ for attack removals.

We can determine whether an extension $E\in\ArgsD$ is stable in $O(n^2m^2)$ time, similarly to the case of the feasibility framework $\AFF$.
In particular, we check if $E$ is conflict-free in $O(n^2m^2)$ time, and if $E$ defends every argument of $\ArgsD\setminus E$ in $O(n^2m^2)$ time. 
\end{proof}

Efficiently capturing feasibility constraints, optimality conditions and fixed decisions in AA allows us to provide tractable explanations refuting or certifying `goodness' of the schedules generated by the optimization solver or else proposed by the user. 
We do this next.

\section{Explanations}
\label{sec:expl}


We here formally define argumentative explanations as to why a given schedule is not `good'. 
We define two types of explanations: 
in terms of \emph{attacks} and \emph{non-attacks} in AFs. 

At a high-level, if a given schedule $S$ is not feasible, efficient or violates fixed decisions, 
the formal argumentative explanations allow to identify which assignments, represented by arguments, are responsible. 
In addition, existence or non-existence of attacks with respect to the identified arguments determines the reasons as represented by the attack relationships of different AFs: feasibility, optimality and fixed decisions. 
Thus identified assignments and reasons allow to instantiate argumentative explanations with templated natural language generated (NLG) explanations to be given to the user, 
e.g.\ as in \cite{Zhong.et.al:2019}. 
Further, if $S$ is `good' as far as AFs can model, 
an NLG explanation can be given relating to the properties satisfied by $S$.

\subsection{Explanations via Attacks}
\label{subsec:att}

\emph{Explanations via attacks} concern schedule feasibility, pairwise exchanges and negative fixed decisions. 
We focus on attacks among arguments in the extension $\argsE$ corresponding to a given schedule $S$. 
These attacks make $\argsE$ non-conflict-free and hence not stable, 
and arise whenever $S$ is not feasible due to some job assigned to more than one machine, 
or violates either PEP or negative fixed decisions. 
We exploit this to define argumentative explanations for why $S$ is not feasible, not efficient or violates fixed decisions.

\begin{definition}
\label{defn:att}
Let $S$ be a schedule, $\argsE \approx S$ and $\AF \in \{ \AFF, \AFS, \AFD \}$. 
We say that \emph{an attack $\arga \attacks \argb$ with $\arga, \argb \in \argsE$ explains why $S$}:
\begin{itemize}
\item \emph{is not feasible}, when $(\arga, \argb) \in \, \attacksF$;
\item \emph{is not efficient}, when $(\arga, \argb) \in \, \attacksS \setminus \attacksF$;
\item \emph{violates fixed decisions}, when $(\arga, \argb) \in \, \attacksD \setminus \attacksF$. 
\end{itemize} 
\end{definition}

So, if a given schedule $S$ 
\begin{enumerate*}[(i)]
\item is either not feasible due to some job assigned to more than one machine ($\attacksF$), 
\item or is not efficient due to some improving pairwise exchange ($\attacksS \setminus \attacksF$), 
\item or assigns some job contrary to a negative fixed decision ($\attacksD \setminus \attacksF$), 
\end{enumerate*}
then the particular reason together with the relevant assignments is indicated. 
This allows to give an NLG explanation via template
\begin{quote}
$S$ $\{$
\begin{itemize*}
\item is not feasible; 
\item is not efficient; 
\item violates fixed decisions
\end{itemize*}$\}$ because 
attack $\arga_{i, j} \attacks \arga_{k, l}$ shows that $\{$
\begin{itemize*}
\item two machines $i$ and $k$ are assigned the same job $j = l$; 
\item $S$ can be improved by swapping jobs $j$ and $l$ on machines $i$ and $k$; 
\item job $i=k$ is assigned to machine $j=l$ contrary to the negative fixed decision $(i, j)$
\end{itemize*}$\}$
\end{quote}
with cases chosen and indices $i, j, k, l$ instantiated accordingly. 
We exemplify this in the three settings next. 

\begin{example}
\label{ex:att feasible}
In Example \ref{ex:feasibility}, let schedule $S$ be given by $x_{1, 1} = x_{2, 1} = 1$ and $x_{1, 2} = x_{2, 2} = 0$.
$S$ is not feasible, because job 1 is assigned to 2 machines (e.g.\ nurses).
We have $S \approx \argsE = \{ \arga_{1, 1}, \arga_{2, 1} \}$.
Any of the attacks $\arga_{1, 1} \attacks \arga_{2, 1}$ and $\arga_{2, 1} \attacks \arga_{1, 1}$
in the feasibility AF $\AFF = \AF$ explains why $S$ is not feasible:
see Figure \ref{fig:graphs}a).
One NLG explanation is:
$S$ is not feasible because
attack $\arga_{1, 1} \attacks \arga_{2, 1}$ shows that
two machines (e.g.\ nurses) $1$ and $2$ are assigned the same job $1$.
\end{example}

\begin{example}
\label{ex:att optimal}
In Example \ref{ex:optimality}, 
the attack $\arga_{2, 3} \attacks \arga_{1, 2}$ in the optimality AF $\AFS = \AF$ 
explains why $S \approx \{ \arga_{1, 1}, \arga_{1, 2}, \arga_{2, 3} \}$ 
is not efficient, particularly as it violates PEP:
see Figure \ref{fig:graphs}b).
The NLG explanation is:
$S$ is not efficient because 
attack $\arga_{2, 3} \attacks \arga_{1, 2}$ shows that 
$S$ can be improved by swapping jobs $3$ and $2$ between machines (e.g.\ nurses) $2$ and $1$.
\end{example}

\begin{example}
\label{ex:att fixed}
In Example \ref{ex:fixed}, the self-attack $\arga_{2, 2} \attacks \arga_{2, 2}$
in the fixed decision AF $\AFD = \AF$
explains why $S \approx \{ \arga_{1, 1}, \arga_{2, 2} \}$ violates the negative fixed decision $D^-$ represented by the self-attack:
see Figure \ref{fig:graphs}c).
The NLG explanation is:
$S$ violates fixed decisions because
attack $\arga_{2, 2} \attacks \arga_{2, 2}$ shows that
job $2$ is assigned to machine (e.g.\ nurse) $2$ contrary to the negative fixed decision $(2, 2)$.
\end{example}

\subsection{Explanations via Non-Attacks}
\label{subsec:natt}

\emph{Explanations via non-attacks} concern schedule feasibility, single exchanges and positive fixed decisions. 
We here focus on arguments outside the extension which are not attacked by the extension $\argsE$ corresponding to a given schedule $S$. 
Such non-attacks result in $\argsE$ being not stable, 
and arise whenever $S$ is not feasible due to some unassigned job, 
or violates either SEP or positive fixed decisions. 
As in the case of explanations via attacks, we exploit this to define argumentative explanations for why $S$ is not feasible, not efficient or violates fixed decisions. 

\begin{definition}
\label{defn:natt}
Let $S$ be a schedule, $\argsE \approx S$ and $\AF \in \{ \AFF, \AFS, \AFD \}$. 
We say that \emph{a non-attack $\argsE \nattacks \argb$ with $\argb \not\in \argsE$ explains why $S$}:
\begin{itemize}
\item \emph{is not feasible}, when $\attacks \, = \attacksF$;
\item \emph{is not efficient}, when $\attacks \, = \attacksS$ and $\argb \attacksS \argsE$;
\item \emph{violates fixed decisions}, when $\attacks \, = \, \attacksD$ and $\argb$ is unattacked. 
\end{itemize} 
\end{definition}

As with explanations via attacks, if a given schedule is not `good', 
then the particular reason together with the relevant assignments is indicated. 
This allows to give an NLG explanation via template
\begin{quote}
$S$ $\{$
\begin{itemize*}
\item is not feasible; 
\item is not efficient; 
\item violates fixed decisions
\end{itemize*}$\}$ because 
non-attack $\argsE \nattacks \arga_{k, l}$ shows that $\{$
\begin{itemize*}
\item job $l$ is not scheduled; 
\item $S$ can be improved by moving job $l$ to machine $k$; 
\item job $l$ is not assigned to machine $k$ contrary to the positive fixed decision $(k, l)$
\end{itemize*}$\}$
\end{quote}
with cases chosen and indices $i, j, k, l$ instantiated accordingly. 
We exemplify this in the three settings next. 

\begin{example}
\label{ex:natt feasible}
In Example \ref{ex:feasibility}, let schedule $S$ be given by $x_{1, 1} = 1$ and $x_{1, 2} = x_{2, 1} = x_{2, 2} = 0$.
We have $S \approx \argsE = \{ \arga_{1, 1} \}$.
Both non-attacks $\argsE \nattacks \arga_{1, 2}$ and $\argsE \nattacks \arga_{2, 2}$
in the feasibility AF $\AFF = \AF$
explain why $S$ is not feasible:
see Figure \ref{fig:graphs}d).
One NLG explanation is:
$S$ is not feasible because
non-attack $\argsE \nattacks \arga_{1, 2}$ shows that
job $2$ is not scheduled.
\end{example}

\begin{example}
\label{ex:natt optimal}
In Example \ref{ex:optimality},
the non-attack $\argsE \nattacks \arga_{2, 1}$ in the optimality AF $\AFS = \AF$
explains why $S$ is not efficient, as it violates SEP:
see Figure \ref{fig:graphs}b).
The NLG explanation is:
$S$ is not efficient because
non-attack $\argsE \nattacks \arga_{2, 1}$ shows that the longest completion (e.g.\ shift) time can be reduced by moving job $1$ to machine (e.g.\ nurse) $2$.
\end{example}

\begin{example}
\label{ex:natt fixed}
In Example \ref{ex:fixed}, let schedule $S$ be given by $x_{1, 2} = x_{2, 1} = 1$, $x_{1, 1} = x_{2, 2} = 0$.
Then $\{ \arga_{1, 2}, \arga_{2, 1} \} = \argsE \approx S$.
In the fixed decision AF $\AFD = \AF$, 
the non-attack $\argsE \nattacks \arga_{1, 1}$
explains why $S$ violates the positive fixed decision $D^+$ represented by the unattacked argument $\arga_{1, 1}$:
see Figure \ref{fig:graphs}e).
The NLG explanation is:
$S$ violates fixed decisions because
non-attack $\argsE \nattacks \arga_{1, 1}$ shows that
job $1$ is not assigned to machine (e.g.\ nurse) $1$ contrary to the positive fixed decision $(1, 1)$.
\end{example}

\subsection{Desiderata for \argopt}
\label{subsec:properties}

We now show that our argumentative explanations meet the desiderata stated in Section \ref{sec:problem}.

\begin{theorem}
\label{thm:expl}
Let $S$ be a schedule, $\argsE \approx S$ and $\AF \in \{ \AFF, \AFS, \AFD \}$. 
\begin{itemize}
\item $S$ is not feasible / is not efficient / violates fixed decisions, respectively, iff:
either there is an attack $\arga \attacks \argb$ with $\arga, \argb \in \argsE$ 
or there is a non-attack $\argsE \nattacks \argb$ with $\argb \not\in \argsE$, 
explaining why $S$ is not feasible / is not efficient / violates fixed decisions, respectively. 
\item Explaining why $S$ is not feasible / is not efficient / violates fixed decisions can be done in $O(n^2 m^2)$ time. 
\item Each explanation is polynomial in the size of $S$.
\end{itemize}
\end{theorem}

\begin{proof}
The first claim follows from Definitions \ref{defn:feasibility af}, \ref{defn:optimality af}, \ref{defn:att} and \ref{defn:natt}, 
Theorems \ref{thm:feasibility}, \ref{thm:optimality} and \ref{thm:fixed}. 
The second claim follows from Lemmas \ref{lemma:feasibility polynomial}, \ref{lemma:optimality polynomial} and \ref{lemma:fixed polynomial}. 
The third claim follows from the fact that mapping schedules and extensions one to another (Definition \ref{defn:corresponding}) is linear in $O(nm)$. 
\end{proof}

This result shows that \argopt\ meets the desiderata 
of soundness and completeness, computational and cognitive tractability. 
Theorem \ref{thm:expl} also implies that we can provide explanations if and when the given schedule $S$ is `good'.
Indeed, if $S$ is feasible / efficient / satisfies fixed decisions, 
then the corresponding extension $\argsE$ is a \emph{certificate} in the feasibility / optimality / fixed decision AF to the `goodness' of $S$. 
This certificate can help the user understand the accompanying NLG explanations as to why the schedule is `good'. 

For instance, consider the fixed decision AF and its unique stable extension $\argsE = \{ \arga_{1, 1}, \arga_{1, 2} \}$ as in Figure \ref{fig:graphs}c). 
$\argsE$ certifies that schedule $S \approx \argsE$ where nurse 1 does both jobs 1 and 2 is feasible and meets the fixed decisions: 
nurse 1 is assigned job 1 because of e.g.~a manager request, 
as per the positive fixed decision $(1, 1)$ represented by the unattacked argument $\arga_{1, 1}$; 
similarly, nurse 1 is assigned job 2 because e.g.~nurse 2 is unqualified, 
as per the negative fixed decision $(2, 2)$ represented by the self-attacking argument $\arga_{2, 2}$. 

\section{Related Work}
\label{sec:related}


To the best of our knowledge, there are no works concerning either explainable scheduling or integrating argumentation and optimization to explain the latter. 
Some preliminary works consider explainable planning, e.g.~\cite{Fox:Long:Magazzeni:2017}, 
which is generally different from scheduling. 
Argumentation can also be used for making and explaining decisions, 
e.g.~\cite{Amgoud:Prade:2009,Zeng.et.al:2018-AAMAS}, 
but mainly in multicriteria decision making which is a different setting from ours. 
Integration-wise, abduction is used for scheduling in e.g.~\cite{AppAI2001,AppInt1991}, 
but not for the purpose of explaining. 
Optimization can be used to implement argumentation solvers, 
e.g.~by mapping AFs to constraint satisfaction problems \cite{Bistarelli:Santini:2010}, 
which is opposite to using argumentation to supplement optimization. 

Argumentation-based explanations in the literature are by and large formalized as (sub-)graphs/trees within AFs, see e.g.~\cite{Garcia:Chesnevar:Rotstein:Simari:2013,Fan:Toni:2015,Cyras:Fan:Schulz:Toni:2018,Rago:Cocarascu:Toni:2018,Zeng.et.al:2018-AAMAS}. 
There the user needs to follow the reasoning chains represented by the graphs to deduce the reasons for why a particular argument (representing e.g.~a statement, a decision, a recommendation) is acceptable. 
In contrast, our argumentative explanations consist of at most two decision points (arguments) and the associated relationship ((non-)attack). 
They can thus be seen as paths of length 1 that pinpoint exactly which decisions violate which properties for a given schedule and optimization considerations, 
without the need to follow possibly lengthy chains of arguments. 
Our explanations are thus cognitively tractable. 
They can also be efficiently generated and afford natural language interpretations. 

Other graph-based models could be used to explain decisions of, in particular, machine learning classifiers, 
e.g.~ordered decision diagrams as in \cite{Shih:Choi:Darwiche:2018} 
and decision trees as in \cite{Frosst:Hinton:2017}. 
Moreover, natural language explanations could also be used for explanations, 
e.g.~via counterfactual statements for machine learning predictions \cite{Sokol:Flach:2018}.
We leave the study of relationships and formal comparison to such approaches for future work.

\section{Conclusions and Future Work}
\label{sec:conclusions}

This paper introduces a paradigm for clearly explaining to a user why a proposed schedule 
is `good' or not.
We propose abstract argumentation as an intermediate layer between the user and the optimization solver for defining and extracting explanations. 
In the makespan scheduling problem, we capture three essential dimensions---feasibility, efficiency, fixed decisions---and capture them with argumentation frameworks.
These proposed argumentative explanations justify whether and why a given schedule is `good' in those dimensions. 
We also establish the soundness and completeness of argumentative explanations, 
prove that they can be efficiently extracted 
and show how argumentative explanations can give rise to natural language explanations. 

This work explicitly incorporates an example that assigns jobs to specific nurses. 
Each job is completed exactly once and the goal is that everyone gets to leave work as quickly as possible. 
Our \emph{fixed decision} setting also recognizes that some nurses have (or lack) particular skills and therefore certain nurses cannot be assigned certain jobs. 
But we could incorporate more modeling requirements into this framework, e.g.\ introduce constraints incorporating
contractual obligations such as a number of shifts per week or design the solution so that it is
more robust to uncertainty \cite{letsios2018exact}. 
Moreover, we have shown how to incorporate explanations for some necessary optimality conditions, 
but we could also develop intuitive explanations for other optimality concepts such as fractional relaxations or cutting planes. 
We leave these extensions for future work.

\subsubsection*{Acknowledgements}

The authors were funded by the EPSRC project 
\textbf{EP/P029558/1} ROAD2H, 
except for Dimitrios Letsios who was funded by the EPSRC project \textbf{EP/M028240/1} Uncertainty-Aware Planning and Scheduling in the Process Industries.

\noindent
\textbf{Data access statement}: 
No new data was collected in the course of this research.

\fontsize{9.5pt}{10.5pt}
\selectfont
\bibliography{references}
\bibliographystyle{aaai}

\end{document}